\documentclass[twoside]{article}

\usepackage[arxiv]{aistats2021}
\usepackage[utf8]{inputenc} 
\usepackage[T1]{fontenc}    
\usepackage{url}            
\usepackage{booktabs}       
\usepackage{amsfonts}       
\usepackage{nicefrac}       
\usepackage{microtype}      

\usepackage{silence}
\usepackage{amsmath,amsfonts,amssymb,amsthm,thmtools,dsfont}
\usepackage{mathtools}
\usepackage{xparse}
\usepackage{enumitem} 
\usepackage{etoolbox}
\usepackage{amsfonts}
\usepackage{hhline}
\usepackage{multirow}
\usepackage{algorithm}
\usepackage{tabularx}
\usepackage[all]{foreign}   

\usepackage[colorlinks=true,bookmarks=false,allcolors=blue,breaklinks=true]{hyperref}
\usepackage[noend]{algpseudocode}
\usepackage[nameinlink,noabbrev]{cleveref}

\usepackage[backend=biber,
            sorting=nyt,
            style=authoryear,
            uniquename=true
           ]{biblatex}

\DeclareCiteCommand{\cite}[\mkbibparens]
  {\usebibmacro{prenote}}
  {\usebibmacro{citeindex}%
    \printtext[bibhyperref]{\usebibmacro{cite}}}
  {\multicitedelim}
  {\usebibmacro{postnote}}

\DeclareCiteCommand*{\cite}[\mkbibparens]
  {\usebibmacro{prenote}}
  {\usebibmacro{citeindex}%
    \printtext[bibhyperref]{\usebibmacro{citeyear}}}
  {\multicitedelim}
  {\usebibmacro{postnote}}

\declaretheoremstyle[bodyfont=\itshape,notefont=\bfseries]{thmbf}
\declaretheoremstyle[notefont=\bfseries]{defibf}
\declaretheoremstyle[headfont=\normalfont\itshape,qed=$\square$]{proofita}

\declaretheorem[style=thmbf,numberwithin=section,name={Theorem}]{theorem}
\declaretheorem[style=thmbf,numberlike=theorem,name={Proposition}]{proposition}

\declaretheorem[style=thmbf,numbered=no,name={Corollary}]{corollary*}

\declaretheorem[style=defibf,numberlike=theorem,name={Definition}]{definition}

\declaretheorem[style=defibf,numbered=no,name={Remark}]{remark}

\newcommand*\rnn{\textsc{rnn}}
\newcommand*\rnns{\textsc{rnn}s}

\newcommand*\lstm{\textsc{lstm}}

\newcommand*\rgd{\textsc{rgd}}
\newcommand*\bfgs{\textsc{bfgs}}
\newcommand*\sgd{\textsc{sgd}}
\newcommand*\gd{\textsc{gd}}

\newcommand*\adam{\textsc{adam}}
\newcommand*\adagrad{\textsc{adagrad}}
\newcommand*\amsgrad{\textsc{amsgrad}}
\newcommand*\rmsprop{\textsc{rmsprop}}
\newcommand*\scornn{\textsc{scornn}}
\newcommand*\scurnn{\textsc{scurnn}}

\newcommand*\exprnn{\textsc{exprnn}}
\newcommand*\dtriv{\textsc{dtriv}}
\newcommand*\atriv{\textsc{atriv}}
\newcommand*\mnist{\textsc{mnist}}
\newcommand*\pmnist{\textsc{p-mnist}}
\newcommand*\timit{\textsc{timit}}
\newcommand*\mse{\textsc{mse}}

\newcommand*\lu{\textsc{lu}}

\newcommand*\gpu{\textsc{gpu}}
\newcommand*\gpus{\textsc{gpu}s}

\DeclarePairedDelimiter\pa{\lparen}{\rparen}    
\DeclarePairedDelimiter\norm{\lVert}{\rVert}    
\DeclarePairedDelimiter\ceil{\lceil}{\rceil}    

\NewDocumentCommand{\infnorm}{ s O{} m }{%
  \IfBooleanTF{#1}{\norm*{#3}}{\norm[#2]{#3}}_{\infty}%
}
\NewDocumentCommand{\twonorm}{ s O{} m }{%
  \IfBooleanTF{#1}{\norm*{#3}}{\norm[#2]{#3}}_2%
}
\NewDocumentCommand{\tvnorm}{ s O{} m }{%
  \IfBooleanTF{#1}{\norm*{#3}}{\norm[#2]{#3}}_{\textup{TV}}%
}
\NewDocumentCommand{\onenorm}{ s O{} m }{%
  \IfBooleanTF{#1}{\norm*{#3}}{\norm[#2]{#3}}_1%
}
\NewDocumentCommand{\frobnorm}{ s O{} m }{%
  \IfBooleanTF{#1}{\norm*{#3}}{\norm[#2]{#3}}_F%
}
\NewDocumentCommand{\scalar}{s O{} >{\SplitArgument{1}{,}}m}{%
    \IfBooleanTF{#1}{\scalaraux*#3}{\scalaraux[#2]#3}%
}
\DeclarePairedDelimiterX{\scalaraux}[2]{\langle}{\rangle}{#1, #2}

\newcommand*\circledaux[1]{\tikz[baseline=(char.base)]{
    \node[shape=circle,draw,inner sep=0.8pt] (char) {#1};}}

\NewDocumentCommand{\circled}{ m o }{%
    \IfNoValueTF{#2}{ \circledaux{#1} }{ \stackrel{\circledaux{#1}}{#2} }%
}

\DeclareMathOperator*{\cay}{cay}                        
\newcommand*{\Hol}{\mathrm{Hol}}                        
\newcommand*\dif{\mathrm{d}}                            
\newcommand*\Id{\mathrm{Id}}                            
\newcommand*\grad{\nabla}                               
\newcommand\conn{\nabla}                                
\newcommand*\tensor{\otimes}                            
\newcommand*\defi{\coloneqq}                            
\newcommand*\iso{\cong}                                 
\newcommand{\code}{\texttt}                             
\newcommand*\gm{\textsl{g}}                             
\newcommand*\RR{\mathbb{R}}                              

\newcommand*\PP{\mathbb{P}}                                
\newcommand*\PPb{\widehat{\mathbb{P}}}                    

\newcommand\transaux{\intercal}                             
\newcommand\trans[1]{#1^\transaux}                          
\DeclareMathOperator\Hess{Hess}                             
\DeclareMathOperator\diam{diam}                             
\newcommand\dual[1]{#1^\ast}                                
\newcommand\MM{\mathcal{M}}                                 
\newcommand\glie{\mathfrak{g}}                              
\newcommand\I{\mathrm{I}}                                   
\newcommand\HH{\mathcal{H}}                                 

\newcommand{\lfrac}[2]{\mbox{\small$\displaystyle\frac{#1}{#2}$}}

\let\phi\varphi
\let\epsilon\varepsilon
\let\subset\subseteq

\newcommand*\deffun[1]{\dodeffunction#1\relax}
\def\dodeffunction#1:#2->#3;#4\relax
    {\ifblank{#4}
    {#1\colon#2\to#3}
    {\!\begin{aligned}#1\colon#2&\to#3\\ \dodeffunctionaux#4\relax\end{aligned}}}
\def\dodeffunctionaux#1->#2\relax{#1&\mapsto#2}

\catcode`\|=\active
\DeclarePairedDelimiterX\set[1]{\lbrace}{\rbrace}
  {\mathcode`\|="8000 \def|{\:\delimsize\vert\:}#1}
\catcode`\|=12 %

\NewDocumentCommand{\enorm}{ s O{} m }{%
    \IfBooleanTF{#1}{\norm*{#3}}{\norm[#2]{#3}}_{\E}%
}
\NewDocumentCommand{\denorm}{ s O{} m }{%
    \dual{\IfBooleanTF{#1}{\enorm*{#3}}{\enorm[#2]{#3}}}%
}

\DeclareMathOperator{\Uaux}{U}                           
\DeclareMathOperator{\SLaux}{SL}                         
\DeclareMathOperator{\GLaux}{GL}                         
\DeclareMathOperator{\SOaux}{SO}                         
\DeclareMathOperator{\SUaux}{SU}                         
\NewDocumentCommand{\U}{ m }{ \Uaux\pa{#1} }
\NewDocumentCommand{\SL}{ m }{ \SLaux\pa{#1} }
\NewDocumentCommand{\GL}{ m }{ \GLaux\pa{#1} }

\NewDocumentCommand{\SO}{ m }{ \SOaux\pa{#1} }
\NewDocumentCommand{\SU}{ m }{ \SUaux\pa{#1} }

\NewDocumentCommand{\commasaux}{ m m }{%
    \IfNoValueTF{#2}{ #1 }{ #1, #2 }%
}
\NewDocumentCommand{\prodaux}{ m m }{%
    \IfNoValueTF{#2}{ #1 \times #1 }{ #1 \times #2 }%
}

\makeatletter
\renewcommand\paragraph{\@startsection{paragraph}{4}{\z@}%
                                    {0ex \@plus0.5ex \@minus.2ex}%
                                    {-1em}%
                                    {\normalfont\normalsize\bfseries}}
\makeatother

\usepackage{todonotes}

\begin{document}

%

\twocolumn[

\aistatstitle{Adaptive and Momentum Methods on Manifolds\\Through Trivializations}

\aistatsauthor{Mario Lezcano Casado}

\aistatsaddress{\url{mario.lezcanocasado@maths.ox.ac.uk}\\Mathematical Institute. University of Oxford} ]

\begin{abstract}
Adaptive methods do not have a direct generalization to manifolds as the adaptive term is not invariant.
Momentum methods on manifolds suffer from efficiency problems stemming from the curvature of the manifold.
We introduce a framework to generalize adaptive and momentum methods to arbitrary manifolds by noting that for every differentiable manifold, there exists a radially convex open set that covers almost all the manifold.
Being radially convex, this set is diffeomorphic to $\RR^n$.
This gives a natural generalization of any adaptive and momentum-based algorithm to a set that covers almost all the manifold in an arbitrary manifolds.
We also show how to extend these methods to the context of gradient descent methods with a retraction.
For its implementation, we bring an approximation to the exponential of matrices that needs just of $5$ matrix multiplications, making it particularly efficient on \gpus.
In practice, we see that this family of algorithms closes the numerical gap created by an incorrect use of momentum and adaptive methods on manifolds.
At the same time, we see that the most efficient algorithm of this family is given by simply pulling back the problem to the tangent space at the initial point via the exponential map.

\end{abstract}

\section{Introduction}
Since the introduction of momentum methods by Nesterov~\cite{nesterov1983method}, and adaptive methods with \adagrad~\cite{duchi2011adaptive}, the deep learning field has seen an eruption of methods tailored for non-convex stochastic optimization.
Some of these methods, like \adagrad{} or \amsgrad~\cite{reddi2018convergence}, offer theoretical guarantees under standard assumptions.
Others, like \rmsprop~\cite{hinton2012lecture} and \adam~\cite{kingma2015adam}, are combinations of the former coupled with heuristics to improve their performance.

At the same time, in the last years, optimization on manifolds has become ever more present in the field of machine learning, in order to perform constrained optimization. Manifolds are the natural object to treat symmetric positive definite matrices in metric learning~\cite{zadeh2016geometric} and Bayesian statistics~\cite{rasmussen2005gaussian}, orthogonal matrices when performing matrix factorization and \rnn{} regularization~\cite{arjovsky2016unitary}, low-rank matrices for low-rank models~\cite{sanyal2020stable}, or invertible matrices in normalizing flows~\cite{papamakarios2019normalizing}.

There has been some work trying to reconcile both approaches. One of them came by considering a surjective map from a Euclidean space onto the manifold $\MM$, allowing to pullback the problem from $\MM$ to $\RR^n$~\cite{lezcano2019trivializations}. A different line of work has generalized these methods to manifolds with a product structure~\cite{becigneul2019riemannian}. However, it is still not clear how to couple momentum and adaptive methods with Riemannian gradient descent on a general manifold.

\paragraph{The problem}
As pointed out in~\cite[][Section $7$]{lezcano2019trivializations}, there is a performance gap between gradient descent-like algorithms (\dtriv$1$) and algorithms based in pulling back the problem to a fixed tangent space (\dtriv$\infty$). This had also been noticed in the literature when using adaptations of \bfgs{} on manifolds. We aim to give an explanation for these problems and offer a gradient descent-like algorithm whose performance matches that of \dtriv$\infty$ without using extra memory.

\paragraph{Summary of the contributions and main idea}
We pinpoint two problems that the adaptive and momentum algorithms suffer when considered on a Riemannian manifold
\begin{enumerate}[itemsep=0mm]
    \item \textbf{Non-trivial tangent bundle.} The adaptive term is not invariant under change of coordinates
    \item \textbf{Non-trivial holonomy.} The momentum term depends on the optimization path
\end{enumerate}

As it will be shown, the first problem completely prevents generalizing adaptive methods to general manifolds. The second does not, but incurs a penalty in practice.
We show that these problems are intrinsic to the topology of the manifold, so they cannot be fixed just by choosing the metric. To circumvent them, we take a more modest approach: We will work on almost all the manifold.

We use a deep result in Riemannian geometry that allows us to work on the manifold as we do on $\RR^n$. This holds on any connected complete Riemannian manifold and as such, it may be implemented for any connected differentiable manifold~\cite{nomizu1961existence}.

The second idea starts by computing all the gradients at a fixed tangent space, as done in~\cite{lezcano2019trivializations}. We then construct a family of linear maps that take steepest descent directions to steepest descent directions to move the gradient with the adaptive and momentum terms incorporated to the tangent space of the current iterate. This will give us a way to move the momentum and adaptive term without incurring in the error coming from holonomy. Geometrically, we are taking advantage of the dense trivial contractible sub-bundle of the tangent bundle given by the previous theorem.

Putting these two ideas together, we are able to generalize adaptive and momentum algorithms like \adam, \adagrad, or \rmsprop{} to manifolds in a way that the practical performance of these algorithm matches that of the previous best algorithms considered in the literature.
We also give a fast implementation of the exponential of matrices in \gpus{} that might be of independent interest.

\paragraph{Limitations}
Even though the resulting algorithm will converge to similar test accuracies, the algorithms proposed make use of the logarithm of matrices to achieve so. As such, even though the proposed algorithms solve the mentioned performance gap, at the moment, regular \dtriv$\infty$---pulling back the problem to a fixed tangent space---is noticeably faster. This may change in the future when faster \gpu{} approximations to the logarithm of matrices become available.

\section{Related Work}\label{sec:related}
\paragraph{Riemannian gradient descent (\rgd).}
This algorithm was first introduced in the Riemannian setting in~\cite{luenberger1972gradient}. The algorithm follows the steepest descent direction along a geodesic, as its Euclidean counterpart. Following its inception, most algorithms proposed for convex, non-convex, stochastic and non-stochastic, and constrained optimization have found analogues in the Riemannian setting~\cite{edelman1998geometry,smith1993geometric,absil2009optimization}. The field is mostly divided in algorithms for geodesically convex optimization~\cite{udriste1994convex,zhang2016first,zhang2016riemannian}, and those based on general retractions for non-convex functions~\cite{absil2009optimization,bonnabel2013stochastic,boumal2019global}.

\paragraph{Trivializations.}
A second approach to optimization on manifolds comes from considering a surjective map from a Euclidean space onto the manifold and using it to pullback the problem to a flat space. This approach has been used to perform first order optimization~\cite{helfrich2018orthogonal,lezcano2019cheap}, gradient-free optimization~\cite{dreisigmeyer2018direct}, and optimization of complex networks~\cite{papamakarios2019normalizing}. A general framework for this approach, together with a connection to Riemannian gradient descent was proposed in~\cite{lezcano2019trivializations}.

\paragraph{Adaptive and momentum methods on manifolds.}
Even though most algorithms on Euclidean spaces have found a counterpart in the Riemannian setting, many times via some bounds on the differential and Hessian of the retraction, or through estimates given by a theorem of Rauch, accelerated momentum methods have remained particularly elusive. The proposed methods transport past gradients along geodesics and try to choose an estimating sequence for the step-sizes. We will elaborate on the main idea behind these methods in~\Cref{sec:problem}.

Current theoretical results on acceleration via momentum methods on manifolds either require starting very close to the optimum~\cite{zhang2018estimate}, or require the solution of a complex differential equation in every step~\cite{liu2017accelerated}, or do not guarantee convergence to the minimum~\cite{alimisis2020practical}. We find a similar situation when we go to higher order. For example, there are many generalizations of \bfgs{} to manifolds by parallel transporting different vectors~\cite{huang2015riemannian,huang2018riemannian}, but their performance in practice does not match that of their Euclidean counterparts.
Adaptive methods have recently been generalized to product manifolds~\cite{becigneul2019riemannian} and (with local convergence) to embedded matrix manifolds~\cite{kasai2019adaptive}, but these methods do not extend to general manifolds. Purely heuristic generalizations have also been proposed, where the adaptive term is parallel-transported~\cite{roy2018geometry,cho2017riemannian}.

\section{The Issues of Momentum and Adaptive Methods on Manifolds}\label{sec:problem}
Consider throughout the rest of the paper a Riemannian manifold $(\MM, \gm)$ which is connected and complete.
\subsection{Momentum Methods}
Momentum methods are a generalization of Nesterov's accelerated gradient descent~\cite{nesterov1983method}. The momentum approach was first proposed in~\cite{rumelhart1986learning}, and consists of using a linear combination of previous gradients and the current gradient as the descent direction. To be able to generalize these methods to manifolds, we need to be able to combine gradients at different points. In order to do so, a Riemannian manifold gives us a natural way of extending a vector to a vector field.
\begin{definition}
    The \textbf{adapted vector field} $v^\ast(x)\in T_x\MM$ of a vector $v \in T_p\MM$ is defined in a normal neighborhood $x \in U_p$ by parallel transporting $v$ along the unique geodesic between $x$ and $p$.
\end{definition}
Local existence of a normal neighborhood follows from the inverse function theorem. As we will see later, these neighborhoods turn out to cover almost all the manifold (\cf, \Cref{thm:properties_exp}).

Using adapted vector fields, under suitable assumptions on the iterates falling in a normal neighborhood, there are two direct ways to generalize momentum methods to manifolds which have been extensively considered in the literature.

\paragraph{Transport the momentum.}
Let $g_k \defi \grad f(x_k) \in T_{x_k}\MM$ be the gradient at the current iterate. We may define the momentum term as $m_k = g_k + \alpha m^\ast_{k-1}(x_k) \in T_{x_k}\MM$ for a mixing constant $\alpha > 0$ and then use $-\eta m_k$ as the descent direction, for a learning rate $\eta > 0$. In other words, we transport the momentum from the previous tangent space to the next one, and we update it there. This method, although direct and cheap to implement, is just an approximation to the following variation.

\paragraph{Choose the best approximation for the gradients.}
This approach defines the momentum as $m_k = \sum_{t=1}^k \alpha^{k-t}g^\ast_t(x_k)$. In plain words, we parallel transport all the previous gradients along the shortest geodesic to the current iterate. This can geometrically be interpreted as considering the best approximation at $x_k$ of each of the previous gradients according to the metric that we are considering. Of course, this method has the drawback of not being practical, as one has to store all the gradients from previous steps.

In summary, the second approach gives a faithful approximation of momentum methods with respect to the metric on the manifold at the expense of having to save all the previous iterations. This not practical, so in practice mostly the first approximation has been studied.

The relation between these two methods comes from the holonomy of the parallel transport system.
Holonomy can be thought of as the integral of the curvature, as formalized by the Ambrose-Singer theorem~\cite{ambrose1953theorem}, and it measures how much a vector rotates when parallel-transported around a loop.
As such, the difference between the momentum from these two algorithms is dependent on the optimization path through the holonomy of the parallel transport system.
In $\RR^n$, these two algorithms coincide, and thus, one can get the benefits of not having to store all the previous gradients at no extra price, but that is not the case on a curved manifold.
On a general manifold, the first algorithm is just an approximation of the second one.
Furthermore, this is a topological problem of the manifold, so it cannot be fixed by just changing the metric. We move the exploration of this idea to~\Cref{sec:problems_further}, as it requires more advanced geometric tools.
The practical implications of this approximation were first noted in~\cite{becigneul2019riemannian}, where report that the holonomy breaks the sparsity of the gradients. In practice, as we mentioned in the introduction, one can see the efficiency problems caused by this approximation both in~\cite{lezcano2019trivializations}
and in the experiments section~\Cref{sec:experiments}, where we compare the algorithms in~\cite{lezcano2019trivializations}, which account for an approximation transporting the momentum with the corrected ones that we propose later in this paper. 

\subsection{Adaptive Methods}
Adaptive methods were introduced
by Duchi, Hazan and Singer with
\adagrad~\cite{duchi2011adaptive}. They add a regularization to the size of the coordinates of the gradient coming from optimizing a certain upper bound on the regret of an online algorithm. It accounts to forming a matrix $H_k \defi \sum_{t=1}^k g_t\trans{g_t}$ and regularizing the gradient as $\hat{g}_k = H_k^{-1/2}g_k$, to use $-\eta\hat{g}_t$ as the descent direction. In the original paper, the authors show that the matrix $H_t$ can be approximated by its diagonal, being able then to effectively store it and compute its inverse, as the product $H_t^{-1/2}g_t$ becomes a point-wise division by the square root of the diagonal terms of $H_t$.

The issue with generalizing adaptive methods to manifolds does not come from an efficiency point of view as it happened with momentum methods, but from an intrinsic topological problem: The adaptive term cannot be defined on a manifold without storing $\mathcal{O}(n^2)$ parameters. The product $g_k \trans{g_k}$ is invariantly described on a manifold as the $(1,1)$-tensor given by the tensor product $g_k \tensor \pa{\dif f}_{x_k}$. The problem comes when considering the diagonal of this tensor. Taking the diagonal is not invariant under changes of coordinates, and as such, the adaptive term cannot be defined globally on a manifold without global coordinates---it may be defined only on $\RR^n$.

\section{Adaptive Methods Through Trivializations}
We will introduce the general method that solves the two problems described in~\Cref{sec:problem} in two steps.
We will first show how one can completely avoid these two problems by using \emph{static trivializations}.
Then, we show how these ideas can be generalized to the setting of \emph{dynamic trivializations}. This setting has both static trivializations and \rgd{} as particular cases.

\subsection{Static trivializations}\label{sec:static_triv}
    We start by recalling the static trivialization framework~\cite{lezcano2019trivializations}. We state it for case of the Riemannian exponential map to make the ideas clearer, although it can be performed with more general maps. For a fixed $p \in \MM$, static trivializations pullback the problem from $\MM$ to $T_p\MM$, going from a constrained problem on $\MM$ to an unconstrained problem on $T_pM (\iso \RR^n)$
\[
    \min_{x \in \MM} f(x) \qquad \xrightarrow{\text{trivialization}}\qquad \min_{v \in T_pM} f(\exp_p(v)).
\]

As proved in~\cite{lezcano2019trivializations}, the two problems are equivalent to a first order in the region in which $\exp_p$ is a local diffeomorphism. In this region, static trivializations account for a change of metric in $\MM$ and solving the second problem via gradient descent is equivalent to solving the first one using \rgd{} with respect to a different metric.

    On manifolds with non-positive sectional curvature, by the Cartan-Hadamard theorem, $\exp_p$ is a local diffeomorphism around any point in $T_p\MM$, so $\exp_p$ acts globally as a change of metric. In the general case, we cannot get the result on all the manifold, but we can still get it on almost all of it.
    \begin{theorem}\label{thm:properties_exp}
        Let $(\MM, \gm)$ be a complete and connected Riemannian manifold. Let $\overline{U}_p$ be the (closed) set of vectors $v \in T_p\MM$ for which $\gamma(t) = \exp_p(tv), t \in [0,1]$ is length minimizing, and let $U_p$ be its interior and $C_p$ its boundary. Then, $\exp_p$ is a diffeomorphism on $U_p$ and $\exp_p(C_p)$ has measure zero with respect to the Borel measure induced by $\gm$.
    \end{theorem}
    \vspace{-0.3cm}
    \begin{proof}
        See~\cite[Section 5.7.3]{petersen2016riemannian} for the first part. The second follows from~\cite{itoh1998dimension}.
    \end{proof}
    \vspace{-0.3cm}
    This theorem says that $\exp_p$ is a diffeomorphism on almost all the manifold, and as such, $\exp_p$ does not create critical points on a large enough neighborhood of $p$ diffeomorphic to $\RR^n$. Furthermore, as $T_p\MM \iso \RR^n$, it is clear that this approach readily generalizes momentum and adaptive methods to almost all the manifold (or in non-positive curvature to all the manifold) as one may use \adam{} or \adagrad{} to optimize the function $\deffun{f \circ \exp_p : T_p\MM -> \RR;}$.

\subsection{Dynamic trivializations}\label{sec:dyn_triv}
We now recall the idea behind \emph{dynamic trivializations}~\cite{lezcano2019trivializations} as a connection between \rgd{} and static trivializations, and we use this connection to generalize momentum and adaptive methods to \rgd.

Dynamic trivializations arise naturally from the following observation:
In~\Cref{sec:static_triv}, we considered the Riemannian exponential $\exp_p$ at a fixed point $p \in \MM$ as a trivialization, but the choice of the point $p$ is completely arbitrary. Dynamic trivializations repeatedly change the trivialization point $p$. They take $K$ steps in $T_{p_i}\MM$ and after the $K$-th iterate $v_{i,K} \in T_{p_i}\MM$, they change the trivialization point to $p_{i+1} \defi \exp_{p_i}(v_{i,K})$ and set $v_{i+1,0} = 0$. In other words, they set the new trivialization point to the current iterate on the manifold and start optimizing on its tangent space.

Dynamic trivializations generalize both static trivializations and \rgd. For $K = \infty$---\ie, never changing the trivialization---dynamic trivializations are exactly static trivializations at $\exp_{p_0}$. For $K = 1$, dynamic trivializations recover \rgd~\cite{lezcano2019trivializations}.

Adaptive and momentum methods can be implemented through static trivializations. We will now use the fact that \dtriv$K$ interpolates between \rgd{} and static trivializations to generalize adaptive and momentum methods to the setting of \rgd.

As we argued in~\Cref{sec:problem}, the adaptive term is not invariant under changes of coordinates, and as such, it cannot be defined after a change of basis. For this reason, the adaptive term has to always be computed at the same tangent space. This can be done simply as it is done, as in the setting of static trivializations, by computing the gradient of $f \circ \exp_{p_0}$ and computing the momentum and adaptive term at $p_0$ following the rule of \adam{} or \rmsprop{} or any other adaptive optimizer. From this, we get an adapted descent direction $\hat{g}_t \in T_{p_0}\MM$. To transport this vector to $T_{p_i}\MM$, we shall use the map induced by the dynamic trivialization algorithm, as per the following proposition.
    \begin{proposition}\label{prop:vector_field}
        Let $(\MM, \gm)$ be a Riemannian manifold. Let $f$ be any differentiable function $f$ on $\MM$, and let $p_0,p_i \in \MM$ be some points and $v_{i,k} \in T_{p_i}\MM$ with $x_t \defi \exp_{p_i}(v_{i,k})$ such that $p_0, p_i \in U_{x_t}$ with $U_{x_t}$ as in~\Cref{thm:properties_exp}. Then, the map
        \[
            \deffun{\dif\pa{\exp^{-1}_{p_0} \circ \exp_{p_i}}^\ast_{v_{i,k}} : T_{p_0}\MM -> T_{p_i}\MM;}
        \]
        were $(-)^\ast$ is the adjoint, takes directions of steepest descent of $f \circ \exp_{p_0}$ to directions of steepest descent of $f \circ \exp_{p_i}$. Furthermore, it is the only map between $T_{p_0}\MM$ and $T_{p_i}\MM$ with this property.
    \end{proposition}
    \begin{proof}
        The hypothesis on $p_0, p_i$ and~\Cref{thm:properties_exp} assure that the operator $\dif\pa{\exp^{-1}_{p_0} \circ \exp_{p_i}}_{v_{i,k}}$ exists and is an isomorphism of linear spaces. By the chain rule, letting $w_{i,k} \defi \exp^{-1}_{p_0}(x_t)$ we have that
        \[
            \dif \pa{f \circ \exp_{p_i}}_{v_{i,k}} = \dif \pa{f \circ \exp_{p_0}}_{w_{i,k}} \circ \dif \pa{\exp_{p_0}^{-1} \circ \exp_{p_i}}_{v_{i,k}}
        \]
        and taking adjoints
        \begin{align*}
            \grad\pa{f &\circ \exp_{p_i}}(v_{i,k}) =\\
            &\dif\pa{\exp_{p_0}^{-1} \circ \exp_{p_i}}^\ast_{v_{i,k}}\pa{\grad \pa{f \circ \exp_{p_0}}\pa{w_{i,k}}}.\qedhere
        \end{align*}
    \end{proof}

    In other words, we compute the gradient together with the momentum and adaptive term at $T_{p_i}\MM$ and then we transport it to our current tangent space using the linear map that preserves directions of steepest descent. When putting together the previous discussion with this proposition we get the algorithm \atriv$K$ as described in~\Cref{alg:adaptive_triv}. Note that for $K = \infty$, $\atriv\infty$ is exactly the same as $\dtriv\infty$, as one would expect.

    \paragraph{Comparison against the classic generalizations of \rgd.}
    A clear geometric picture arises when comparing \atriv$1$ with the methods presented in~\Cref{sec:problem}. The second method in that section can be interpreted as considering the flat metric from $T_{x_t}\MM$ to transport the previous gradients from $T_{x_i}\MM$ to that tangent space. This method has the problem of having to store all the previous gradients. Here, we substitute the moving metric at $T_{x_t}\MM$ by the use of a fixed one at $T_{x_0}\MM$, hence being able to accumulate the gradients. Note also the first algorithm introduced in that section allows implementing momentum method at the expense of suffering from the holonomy problems noted in~\cite{becigneul2019riemannian}. Adaptive trivializations do not suffer these problems since they use a metric with trivial holonomy. Furthermore, \atriv$K$ allows for being coupled with adaptive methods, while the classic generalizations do not.

    \paragraph{A note on general retractions.}
    The case for a general retraction is analogue and we have spared the reader the details for concreteness. For every retraction and every $p \in \MM$, we know that there exists a neighborhood $U_p$ such that $r_p$ is invertible, although in general it will not be as large as the one given by the exponential. For this reason, we may be able to perform this algorithm just locally. If the optimization process leaves that neighborhood we can apply a restart process on the momentum and adaptive term. These restarts are rather common in optimization algorithms with momentum (\cf, \cite{allen2017linear}). Solutions of this flavor had already been considered in the literature~\cite{sato2019riemannian}.

\begin{algorithm}[!htp]
    \caption{Adaptive dynamic trivialization}
    \label{alg:adaptive_triv}
    \begin{algorithmic}[1]
            \Require $K > 0$, $p_0 \in \MM$. An adaptive optimizer (\eg, \adam, \rmsprop, \textellipsis)
            \For {$i = 0, \ldots$}
            \State $v_{i, 0} = 0$
                \For {$k = 0, \dots, K-1$}\Comment{Optimize on $T_{p_i}\MM$}
                \State $w_{i,k} = \pa{\exp^{-1}_{p_0} \circ \exp_{p_i}}(v_{i, k})$
                \State $g_t = \grad \pa{f \circ \exp_{p_0}} \pa{w_{i, k}}$
                \State Compute $\hat{g}_t$ using the optimizer, aggregating the adaptive term on $T_{p_0}\MM$
                \State $\tilde{g}_t = \dif\pa{\exp_{p_0}^{-1} \circ \exp_{p_i}}^\ast_{v_{i,k}}(\hat{g}_t)$
                \State $v_{i, k+1} = v_{i, k} - \eta\tilde{g}_t$
                \EndFor
                \State $p_{i+1} = \exp_{p_i}(v_{i,K})$ \Comment{Update trivialization point}
            \EndFor
    \end{algorithmic}
\end{algorithm}

\paragraph{Convergence guarantees}
In the case when the optimizer is gradient descent, we can give exact guarantees for the algorithm by using first and second order bounds on the Riemannian exponential.

\begin{theorem}
    Let $(\MM, \gm)$ be a connected and complete Riemannian manifold with zeroth and first order curvature bounds $\delta \leq \sec \leq \Delta$ and $\norm{\conn R} \leq \Lambda$. Consider a connected submanifold $\mathcal{X}$ with $\diam(\mathcal{X}) = r$ where $r \leq \frac{\sqrt{2}\pi}{\sqrt{\delta + \Delta}}$ if $\Delta > - \delta$, and let $f$ be a function on $\mathcal{X}$ bounded below by $f^\ast \in \RR$ such that $\norm{\Hess f} \leq \alpha$. Then, there exists a constant $\widehat{\alpha}_r$ that does not depend on the dimension of $\MM$ such that $\norm{\Hess \pa{f \circ \exp_p}\vert_{\exp_p^{-1}\pa{\mathcal{X}}}} \leq \widehat{\alpha}_r$ for any $p \in \mathcal{X}$. Therefore, running \atriv$K$ with \gd{} and $\eta = \frac{1}{\widehat{\alpha}_r}$, if all the iterates stay in $\mathcal{X}$, the algorithm will find a point $v_{i, k} \in T_{p_i}\MM$ such that $\norm{\grad \pa{f \circ \exp_{p_i}}(v_{i,k})} < \epsilon$ in at most
        $\ceil[\Big]{\lfrac{2\widehat{\alpha}_r}{\epsilon^2}\pa{f(p_0) - f^\ast}}$
    steps.
\end{theorem}
\begin{proof}
    The result follows from~\cite[Proposition $5.1$]{lezcano2020curvature} and the classical theory of non-convex optimization~\cite{nesterov2004introductory}.
\end{proof}
The bound $\widehat{\alpha}_r$ is an explicit function of $r, \delta, \Delta, \Lambda$ and takes a very simple form for most manifolds used in optimization. For example, when $\MM$ is the special orthogonal group $\SO{n}$ embedded in $\RR^{n \times n}$, it takes the form $\widehat{\alpha}_r = (1+\frac{r}{3})\alpha$. In other words, this result recovers the $\mathcal{O}(\epsilon^{-2})$ convergence rate of the Euclidean case modulo a curvature term and a linear size term.

The hypothesis about the iterates stay in $\mathcal{X}$, although restrictive, is standard when dealing with manifolds~\parencite{bonnabel2013stochastic,zhang2016riemannian,sato2019riemannian,tripuraneni2018averaging,ahn2020nesterov}.

%
\paragraph{Faster exponential of matrices.}
The exponential of matrices is fundamental when dealing with the exponential map in optimization on manifolds.
On the other hand, implementing correctly the exponential of matrices is not an easy task~\cite{moler1978nineteen,moler2003nineteen}. The current standard implementation in Scipy is based on the paper~\cite{al2009new}. This approximation uses an \lu{} decomposition and the solution of a triangular system, which are sequential algorithms.

In order to have a fast implementation in \gpu{} of the exponential of matrices, we implement one that is based on matrix multiplications---a polynomial approximation. The first algorithm of this kind with numerical accuracy guarantees was presented in the recent work by~\cite{bader2019computing}. Their algorithm consists of the scale-squaring trick with a clever evaluation of Taylor approximants. This evaluation, requires $5$ matrix multiplications, making the exponential of matrices a very cheap function use in the context of deep learning. We helped PyTorch maintainers to add this algorithm to PyTorch, getting a $14\times$ wall-clock speed improvement over the Tensorflow implementation. It has been released in PyTorch 1.7.0.

%
%

\section{Experiments}\label{sec:experiments}
\subsection{Experiment set-up}
In this section, we compare the practical performance of the approach presented in this paper with the state-of-the-art algorithms for stochastic Riemannian optimization with orthogonal constraints for neural networks. In particular, we compare it with dynamic trivializations \cite[\dtriv ][]{lezcano2019trivializations}, exponential trivialization~\cite[\exprnn][]{lezcano2019cheap}, Cayley trivializations~\cite[\scornn][]{helfrich2018orthogonal} and Riemannian gradient descent~\cite[\rgd][]{wisdom2016full}. We also include an \lstm{} for reference.

\paragraph{Relations between these methods}
All these methods fall under the dynamic trivialization framework. In particular, there are some relations between them: \rgd{} is \dtriv{}$1$ using \sgd{} as the optimizer, \exprnn{} is \dtriv{}$\infty$ using the identity matrix as the basis for the trivialization, and \atriv$K$ is \dtriv$K$ but with the momentum and adaptive terms in the optimizer corrected. \atriv$\infty$ and \dtriv$\infty$ are the same algorithm. We summarize these relations in~\Cref{tab:algs}.

\begin{table}[!ht]
    \centering
    \caption{Algorithms to optimize with orthogonal constraints}
    \label{tab:algs}
    \begin{tabular}{l|cccc}
        \toprule
        Model & Optimizer & Retraction & Adapt. & Type\\
        \midrule
        \scornn & Any & $\cay_{\I}$ &  Yes & Static\\
        \exprnn & Any & $\exp_{\I}$ &  Yes & Static\\
        \dtriv$\infty$ & Any & $\exp_{B_0}$ & Yes & Static \\
        \midrule
        \rgd      & \sgd & $\cay_B$ & --- & Dyn.\\
        \dtriv$n$ & Any  & $\exp_B$ & No & Dyn.\\
        \atriv$n$ & Any  & $\exp_B$ & Yes & Dyn.\\
        \bottomrule
    \end{tabular}
\end{table}

In the Adapt.\ column in~\Cref{tab:algs}, we note which models treat adaptive methods correctly. We have written $\phi_B(A) = B\phi(B^{-1}A)$ where $\phi = \exp$ or $\phi = \cay$, $A \in T_B\SO{n}$, $B \in \SO{n}$.

The model on which we perform the experiments is a vanilla \rnn{} together with an orthogonal constraint on the recurrent kernel. The orthogonality of $B$ is achieved following the methods in~\Cref{tab:algs}. The non-linearity used is the \code{modrelu}, introduced for these models in~\cite{arjovsky2016unitary}.

We include an implementation of the general framework in the accompanying material. We describe in detail all the hyperparameters needed to replicate the experiments in~\Cref{sec:hyper}.

\paragraph{Orthogonal \rnns{} in practice}
    In practice, we observed that these models benefit from a lower mixing constant in the adaptive term than the default used in \adam. In particular, a mixing constant of $\beta_2 = 0.99$ or even $\beta_2 = 0.95$ improved the stability of the algorithms. This makes a lot of sense since, even though the recurrent layer is orthogonal and the non-linearity is piecewise linear, so that the recurrent gradients are bounded, the input layer might still make the gradients unstable at times. Previous papers reported a better performance of \rmsprop{} over \adam{}~\cite{arjovsky2016unitary,helfrich2018orthogonal}. This is explained by the fact that the mixing constant of \rmsprop{} is initialized by default to $0.99$, while that of \adam{} is initialized to $0.999$. We also observed that using these lower mixing constants, one can almost duplicate the learning rate on the orthogonal parameters.

    We do not make use of these observations to improve our results in the experiments section, as we are not trying to get the best result, but merely to compare the different algorithms to optimize with orthogonal constraints.

We test these algorithms on the two standard problems that the community has been using to test stochastic optimization with orthogonal constraints in the context of neural networks, namely sequential \mnist{} and \timit~\cite{arjovsky2016unitary,henaff2016recurrent,helfrich2018orthogonal,lezcano2019cheap,lezcano2019trivializations}. 

\begin{remark}[Copying and adding problems]
We do not include the adding problem or the copying problem~\cite{hochreiter1997long} which are often considered, as these problems are too simple for recurrent models~\cite{henaff2016recurrent}. All the considered optimization methods converge to a correct answer in the copying and adding problems.
\end{remark}

\subsection{Sequential \mnist}
This task consists on classifying the numbers from $0$ to $9$ present in the \mnist{} dataset~\cite{lecun1998mnist}. This dataset is formed by $60.000$ $28 \times 28$ grayscale images. The sequential task \mnist{} processes each of the $784$ pixels of the image sequentially. We also consider the permuted sequential task \pmnist, were a permutation of the bits is sampled at the beginning of the training and fixed once and for all. This permutation is then used to permute the pixels of all the images before flattening them.

\begin{table}[t]
\centering
\caption{Best test accuracy at \mnist{} and \pmnist{}.}
\label{tab:mnist}
\scshape
\small
\begin{tabular}{l c c c}
    \toprule
    Model & n & \mnist & \pmnist \\
    \midrule
    \midrule
    \atriv{}$1$ & $170$ & $98.2$ & $95.1$ \\
    \dtriv{}$\infty$ & $170$ & $98.1$ & $95.0$ \\
    \dtriv{}$1$ & $170$ & $\mathbf{98.3}$ & $\mathbf{95.2}$ \\
    \exprnn & $170$ & $98.0$ & $94.9$ \\
    \scornn & $170$ & $97.2$ & $94.8$ \\
    \scurnn & $116$ & $97.6$ & $94.9$ \\
    \lstm & $128$ & $81.9$ & $79.5$ \\
    \rgd & $116$ & $94.7$ & $92.5$ \\
    \midrule
    \atriv{}$1$ & $360$ &  $\mathbf{98.9}$ & $96.4$ \\
    \dtriv{}$\infty$ & $360$ &  $\mathbf{98.9}$ & $\mathbf{96.5}$ \\
    \dtriv{}$1$ & $360$ &  $98.4$ & $96.3$ \\
    \exprnn & $360$ & $98.4$ & $96.2$ \\
    \scornn & $360$ & $98.1$ & $95.9$ \\
    \scurnn & $250$ & $98.3$ & $96.2$ \\
    \lstm & $256$ & $88.8$ & $88.8$ \\
    \rgd & $256$ & $96.1$ & $93.9$ \\
    \midrule
    \atriv{}$1$ & $512$ & $\mathbf{99.0}$ & $96.7$ \\
    \dtriv{}$\infty$ & $512$ & $\mathbf{99.0}$ & $\mathbf{96.8}$ \\
    \dtriv{}$1$ & $512$ & $98.7$ & $96.7$ \\
    \exprnn & $512$ & $98.7$ & $96.6$ \\
    \scornn & $512$ & $98.2$ & $96.5$ \\
    \lstm & $512$ & $91.9$ & $91.8$ \\
    \rgd & $512$ & $97.3$ & $94.7$ \\
    \bottomrule
\end{tabular}
\end{table}

We present the results of this experiments in~\Cref{tab:mnist}, where we have denoted by \textsc{N} the size of the hidden layer. This is chosen so that all the models in each group have the same number of parameters.

Before analyzing the results, we recall that \atriv$1$ is exactly the same algorithm as \dtriv$1$ but with a correct treatment of the adaptive term as per~\Cref{alg:adaptive_triv}.
\dtriv$\infty$ and \atriv$1$ both handle correctly the adaptive term, and they both are algorithms of the same family, since \dtriv$\infty$ is the same algorithm as \atriv$\infty$. In general, we observed a consistent performance across the algorithms of the family \atriv$K$ when varying $K$.

In this table we see two phenomena. First, we see that the corrected treatment of the adaptive term in \atriv$1$ gives a consistent improvement over \dtriv$1$ of this algorithm making it match the performance of its static counterpart \dtriv$\infty$. Furthermore, we see that all the dynamic trivializations yield a comparable performance when treating the adaptive term correctly, that is, \dtriv$\infty$ and \atriv$1$ yield a similar perform. This hints that there is no real improvement when it comes to \dtriv$\infty$ or \atriv$1$. As such, it would be recommendable to use \dtriv$\infty$, since by not requiring to change the basis, its implementation is simpler and it is considerably faster.

\subsection{\timit{} speech dataset}
The \timit{} dataset is composed by real-world speech recording~\cite{garofolo1993darpa}. This is a regression task in which, given the $n$ first points of a sequence, we want to predict the point $n+1$. The sequences of audio are preprocessed by downsampling them to $8$kHz. Then, they are represented in a log-scale after applying a short-time Fourier transform. This process generates variable-length sequences (with length between $61$ and $490$) of $129$ complex numbers at every step. This experiment and set-up was first proposed in~\cite{wisdom2016full}.
We present the results of this experiment in~\Cref{tab:timit}.

\begin{table}[t]
\centering
\caption{Test \mse{} at the end of the epoch with the lowest validation \mse{} for the \timit{} task.}
\label{tab:timit}
\small
\scshape
\begin{tabular}{l c c c c}
    \toprule
    Model & n & Val. \mse{} & Test \mse{}\\
    \midrule
    \midrule
    \atriv{}$1$ & $224$ & $\mathbf{4.52}$ & $\mathbf{4.50}$ \\
    \dtriv{}$\infty$ & $224$ & $4.75$ & $4.71$ \\
    \dtriv{}$1$ & $224$ & $6.55$ & $6.54$ \\
    \exprnn & $224$ & $5.34$ & $5.30$ \\
    \scornn & $224$ & $9.26$ & $8.50$ \\
    \scurnn & $128$ & $9.42$ & $7.23$ \\
    \lstm & $84$ &  $15.42$ & $14.30$ \\
    \rgd & $128$ &  $15.07$ & $14.58$ \\
    \midrule
    \atriv{}$1$ & $322$ & $\mathbf{3.33}$ & $\mathbf{3.29}$ \\
    \dtriv{}$\infty$ & $322$ & $3.39$ & $3.76$ \\
    \dtriv{}$1$ & $322$ & $4.56$ & $4.55$ \\
    \exprnn & $322$ & $4.42$ & $4.38$ \\
    \scornn & $322$ & $8.48$ & $7.82$ \\
    \lstm & $120$ & $13.93$ & $12.95$ \\
    \rgd & $192$ & $15.10$ & $14.50$ \\
    \midrule
    \atriv{}$1$ & $425$ & $\mathbf{1.93}$ & $\mathbf{1.92}$ \\
    \dtriv{}$\infty$ & $425$ & $2.00$ & $1.97$ \\
    \dtriv{}$1$ & $425$ & $4.21$ & $4.17$ \\
    \exprnn & $425$ & $5.52$ & $5.48$ \\
    \scornn & $425$ & $7.97$ & $7.36$ \\
    \scurnn & $258$ & $4.40$ & $3.39$ \\
    \lstm & $158$ & $13.66$ & $12.62$ \\
    \rgd & $256$ & $14.96$ & $14.69$ \\
    \bottomrule
\end{tabular}
\end{table}

In this experiment we observe the same behavior of that present in the \mnist{} task.
We see a consistent improvement of (\atriv$1$) over the version of the algorithm that does not correct the adaptive term (\dtriv$1$).
Furthermore, we also see that the correction in the adaptive term (\atriv$1$) yields a comparable performance to that of using the point at the first tangent space to trivialize (\dtriv$\infty$).

\section{Conclusion and Future Work}\label{sec:conclusion}
In this work we present a generalization of adaptive and momentum-based methods to manifolds by making use of the fact that any manifold $\MM$ has a radially convex submanifold that covers almost all $\MM$ and that is diffeomorphic to a subset of $\RR^n$. As such, the subset that is not on the image of this diffeomorphism has measure zero, and restricting our optimization methods to this open subset serves as a tight approximation to the original problem.

We see that by using the generalization of momentum and adaptive methods given by a flat connection in the interior of the segment domain, we achieve comparable performance in practice to that of pulling back the optimization problem along the Riemannian exponential at a given point. In particular, correcting the momentum and adaptive term in this fashion bridges the performance gap created by na\"ively using adaptive methods on top of dynamic trivializations.

We expect that the map considered here can become a useful alternative to parallel-transporting the momentum or other vectors in other optimization methods on manifolds.

\section*{Acknowledgements}
The work of MLC was supported by the Oxford-James Martin Graduate Scholarship.

\printbibliography

@book{absil2009optimization,
    title={Optimization algorithms on matrix manifolds},
    author={Absil, Pierre-Antoine and Mahony, Robert and Sepulchre, Rodolphe},
    year={2009},
    publisher={Princeton University Press}
}

@article{al2009new,
    title={A new scaling and squaring algorithm for the matrix exponential},
    author={Al-Mohy, Awad H. and Higham, Nicholas J.},
    journal={SIAM Journal on Matrix Analysis and Applications},
    volume={31},
    number={3},
    pages={970--989},
    year={2009},
}

@Article{alimisis2020practical,
    title   = {Practical accelerated optimization on {R}iemannian manifolds},
    author  = {Alimisis, Foivos and Orvieto, Antonio and B{\'e}cigneul, Gary and Lucchi, Aurelien},
    journal = {https://arxiv.org/abs/2002.04144},
    year    = {2020},
}

@InProceedings{allen2017linear,
    title          = {Linear coupling: {A}n ultimate unification of gradient and mirror descent},
    author         = {Allen-Zhu, Zeyuan and Orecchia, Lorenzo},
    booktitle      = {Innovations in Theoretical Computer Science Conference, ITCS},
    pages          = {3:1--3:22},
    year           = {2017},
}

@inproceedings{arjovsky2016unitary,
    title={Unitary evolution recurrent neural networks},
    author={Arjovsky, Martin and Shah, Amar and Bengio, Yoshua},
    booktitle={International Conference on Machine Learning, ICML},
    pages={1120--1128},
    year={2016}
}

@Article{arsigny2006geometric,
    author  = {Vincent Arsigny and Pierre Fillard and Xavier Pennec and Nicholas Ayache},
    title   = {Geometric means in a novel vector space structure on symmetric positive-definite matrices},
    number  = {1},
    pages   = {328--347},
    volume  = {29},
    journal = {SIAM Journal on Matrix Analysis and Applications},
    year    = {2006},
}

@inproceedings{arsigny2006log,
    title={A log-{E}uclidean framework for statistics on diffeomorphisms},
    author={Arsigny, Vincent and Commowick, Olivier and Pennec, Xavier and Ayache, Nicholas},
    booktitle={International Conference on Medical Image Computing and Computer-Assisted Intervention, MICCAI},
    pages={924--931},
    year={2006},
}

@article{bader2019computing,
    TITLE = {Computing the matrix exponential with an optimized {T}aylor polynomial approximation},
    AUTHOR = {Bader, Philipp and Blanes, Sergio and Casas, Fernando},
    JOURNAL = {Mathematics},
    VOLUME = {7},
    YEAR = {2019},
    NUMBER = {12},
}

@InProceedings{becigneul2019riemannian,
    title     = {{R}iemannian adaptive optimization methods},
    author    = {Gary Becigneul and Octavian-Eugen Ganea},
    booktitle = {International Conference on Learning Representations, ICLR},
    year      = {2019},
}

@article{bonnabel2013stochastic,
    title={Stochastic gradient descent on {R}iemannian manifolds},
    author={Bonnabel, Silv{\`e}re},
    journal={IEEE Transactions on Automatic Control},
    volume={58},
    number={9},
    pages={2217--2229},
    year={2013},
}

@article {boumal2019global,
    TITLE = {Global rates of convergence for nonconvex optimization on manifolds},
    AUTHOR = {Boumal, Nicolas and Absil, Pierre-Antoine and Cartis, Coralia},
    JOURNAL = {IMA Journal of Numerical Analysis},
    VOLUME = {39},
    NUMBER = {1},
    PAGES = {1--33},
    YEAR = {2019},
}

@article {dreisigmeyer2018direct,
    AUTHOR = {Dreisigmeyer, David W.},
    TITLE = {Direct search methods on reductive homogeneous spaces},
    JOURNAL = {Journal of Optimization Theory and Applications},
    VOLUME = {176},
    YEAR = {2018},
    NUMBER = {3},
    PAGES = {585--604},
}

@article{duchi2011adaptive,
    title={Adaptive subgradient methods for online learning and stochastic optimization},
    author={Duchi, John and Hazan, Elad and Singer, Yoram},
    journal={Journal of Machine Learning Research, JMLR},
    volume={12},
    number={Jul},
    pages={2121--2159},
    year={2011}
}

@Article{edelman1998geometry,
    author  = {Edelman, Alan and Arias, Tom{\'a}s A. and Smith, Steven Thomas},
    title   = {The geometry of algorithms with orthogonality constraints},
    number  = {2},
    pages   = {303--353},
    volume  = {20},
    journal = {SIAM Journal on Matrix Analysis and Applications},
    year    = {1998},
}

@Misc{garofolo1993darpa,
    author    = {Garofolo, J. S. and Lamel, L. F. and Fisher, W. M. and Fiscus, J. G. and Pallett, D. S. and Dahlgren, N. L.},
    title     = {{DARPA} {TIMIT} acoustic phonetic continuous speech corpus},
    number    = {LDC93S1},
    publisher = {Linguistic Data Consortium},
    year      = {1993},
}

@InProceedings{helfrich2018orthogonal,
    title     = {Orthogonal recurrent neural networks with scaled {C}ayley transform},
    author    = {Helfrich, Kyle and Willmott, Devin and Ye, Qiang},
    booktitle = {International Conference on Machine Learning, ICML},
    pages     = {1974--1983},
    year      = {2018},
}

@inproceedings{henaff2016recurrent,
    title={Recurrent orthogonal networks and long-memory tasks},
    author={Henaff, Mikael and Szlam, Arthur and LeCun, Yann},
    booktitle={International Conference on Machine Learning, ICML},
    pages={2034--2042},
    year={2016},
}

@Misc{hinton2012lecture,
    title   = {Lecture 6.5. {RMS}prop: {D}ivide the gradient by a running average of its recent magnitude},
    author  = {Hinton, Geoffrey E.},
    url     = {http://www.cs.toronto.edu/~tijmen/csc321/slides/lecture_slides_lec6.pdf},
    note    = {Course: \emph{Neural networks for machine learning}},
    year    = {2012},
}

@article{hochreiter1997long,
    title={Long short-term memory},
    author={Hochreiter, Sepp and Schmidhuber, J{\"u}rgen},
    journal={Neural Computation},
    volume={9},
    number={8},
    pages={1735--1780},
    year={1997},
}

@InProceedings{huang2015riemannian,
    title     = {A {R}iemannian {BFGS} method for nonconvex optimization problems},
    author    = {Wen Huang and Pierre-Antoine Absil and Kyle A. Gallivan},
    booktitle = {Numerical Mathematics and Advanced Applications, ENUMATH},
    pages     = {627--634},
    year      = {2015},
}

@article{huang2018riemannian,
    author    = {Wen Huang and
                 Pierre{-}Antoine Absil and
                 Kyle A. Gallivan},
    title     = {A {R}iemannian {BFGS} method without differentiated retraction for nonconvex
                 optimization problems},
    journal   = {SIAM Journal on Optimization},
    volume    = {28},
    number    = {1},
    pages     = {470--495},
    year      = {2018},
    doi       = {10.1137/17M1127582},
}

@Article{	  itoh1998dimension,
    title		= {The dimension of a cut locus on a smooth {R}iemannian manifold},
    author	= {Itoh, Jin-ichi and Tanaka, Minoru},
    journal	= {Tohoku Mathematical Journal},
    volume	= {50},
    number	= {4},
    pages		= {571--575},
    year		= {1998},
}

@InProceedings{kingma2015adam,
    title     = {Adam: {A} method for stochastic optimization},
    author    = {Kingma, Diederik P. and Ba, Jimmy},
    booktitle = {International Conference on Learning Representations, ICLR},
    year      = {2015},
}

@Book{kobayashi1963foundations,
    title     = {Foundations of differential geometry},
    author    = {Kobayashi, Shoshichi and Nomizu, Katsumi},
    publisher = {Interscience Publishers, John Wiley \& Sons},
    volume    = {I},
    year      = {1963},
}

@InProceedings{lezcano2019cheap,
    title     = {Cheap orthogonal constraints in neural networks: {A} simple parametrization of the orthogonal and unitary group},
    author    = {Lezcano-Casado, Mario and Mart{\'i}nez-Rubio, David},
    booktitle = {International Conference on Machine Learning, ICML},
    pages     = {3794--3803},
    year      = {2019},
}

@inproceedings{lezcano2019trivializations,
    title = {Trivializations for gradient-based optimization on manifolds},
    author = {Lezcano-Casado, Mario},
    booktitle={Advances in Neural Information Processing Systems, NeurIPS},
    pages = {9154--9164},
    year = {2019},
}

@InProceedings{liu2017accelerated,
    author    = {Yuanyuan Liu and Fanhua Shang and James Cheng and Hong Cheng and Licheng Jiao},
    booktitle = {Advances in Neural Information Processing Systems, NeurIPS},
    title     = {Accelerated first-order methods for geodesically convex optimization on {R}iemannian manifolds},
    pages     = {4868--4877},
    year      = {2017},
}

@Article{luenberger1972gradient,
    author  = {Luenberger, David G.},
    title   = {The gradient projection method along geodesics},
    pages   = {620--631},
    volume  = {18},
    journal = {Management Science},
    year    = {1972},
}

@article {moler1978nineteen,
    TITLE = {Nineteen dubious ways to compute the exponential of a matrix},
    AUTHOR = {Moler, Cleve and Van Loan, Charles},
    JOURNAL = {SIAM Review},
    VOLUME = {20},
    YEAR = {1978},
    NUMBER = {4},
    PAGES = {801--836},
}

@article{moler2003nineteen,
    title={Nineteen dubious ways to compute the exponential of a matrix, twenty-five years later},
    author={Moler, Cleve and Van Loan, Charles},
    journal={SIAM Review},
    volume={45},
    number={1},
    pages={3--49},
    year={2003},
}

@Article{nesterov1983method,
    author	= {Nesterov, Yurii},
    title		= {A method for solving the convex programming problem with convergence rate ${O}(1/k^2)$},
    journal	= {Soviet Mathematics Doklady},
    volume	= {27},
    pages		= {372--376},
    year		= {1983},
}

@article{nomizu1961existence,
    title={The existence of complete {R}iemannian metrics},
    author={Nomizu, Katsumi and Ozeki, Hideki},
    journal={Proceedings of the American Mathematical Society},
    volume={12},
    number={6},
    pages={889--891},
    year={1961},
}

@Article{papamakarios2019normalizing,
    author  = {George Papamakarios and Eric T. Nalisnick and Danilo Jimenez Rezende and Shakir Mohamed and Balaji Lakshminarayanan},
    title   = {Normalizing flows for probabilistic modeling and inference},
    journal = {arXiv preprint arXiv:1912.02762},
    year    = {2019},
}

@Book{petersen2016riemannian,
    author    = {Petersen, Peter},
    title     = {{R}iemannian geometry},
    edition   = {Third},
    publisher = {Springer},
    series    = {Graduate Texts in Mathematics},
    year      = {2016},
}

@Book{rasmussen2005gaussian,
    author    = {Rasmussen, Carl Edward and Williams, Christopher K. I.},
    title     = {Gaussian processes for machine learning},
    publisher = {The MIT Press},
    series    = {Adaptive Computation and Machine Learning},
    year      = {2005},
}

@InProceedings{reddi2018convergence,
    author    = {Sashank J. Reddi and Satyen Kale and Sanjiv Kumar},
    booktitle = {International Conference on Learning Representations, ICLR},
    title     = {On the convergence of {A}dam and beyond},
    year      = {2018},
}

@InProceedings{roy2018geometry,
    author    = {Roy, Soumava Kumar and Mhammedi, Zakaria and Harandi, Mehrtash},
    booktitle = {Computer Vision and Pattern Recognition, CVPR},
    title     = {Geometry aware constrained optimization techniques for deep learning},
    year      = {2018},
}

@InProceedings{sanyal2020stable,
    author    = {Amartya Sanyal and Philip H. Torr and Puneet K. Dokania},
    booktitle = {International Conference on Learning Representations, ICLR},
    title     = {Stable rank normalization for improved generalization in neural networks and {GAN}s},
    year      = {2020},
}

@article{sato2019riemannian,
    author    = {Hiroyuki Sato and
                 Hiroyuki Kasai and
                 Bamdev Mishra},
    title     = {Riemannian stochastic variance reduced gradient algorithm with retraction
                 and vector transport},
    journal   = {SIAM Journal on Optimization},
    volume    = {29},
    number    = {2},
    pages     = {1444--1472},
    year      = {2019},
    doi       = {10.1137/17M1116787},
}

@thesis{smith1993geometric,
    title={Geometric optimization methods for adaptive filtering},
    author={Smith, Steven Thomas},
    year={1993},
    type={Ph.D. dissertation},
    institution={Harvard University},
}

@book {udriste1994convex,
    AUTHOR = {Udri\c{s}te, Constantin},
    TITLE = {Convex functions and optimization methods on {R}iemannian manifolds},
    SERIES = {Mathematics and its Applications},
    VOLUME = {297},
    PUBLISHER = {Springer},
    YEAR = {1994},
}

@inproceedings{wisdom2016full,
    title={Full-capacity unitary recurrent neural networks},
    author={Wisdom, Scott and Powers, Thomas and Hershey, John and Le Roux, Jonathan and Atlas, Les},
    booktitle={Advances in Neural Information Processing Systems, NeurIPS},
    pages={4880--4888},
    year={2016}
}

@inproceedings{zadeh2016geometric,
    title     = {Geometric mean metric learning},
    author    = {Zadeh, Pourya and Hosseini, Reshad and Sra, Suvrit},
    booktitle = {International Conference on Machine Learning, ICML},
    pages     = {2464--2471},
    year      = {2016}
}

@InProceedings{zhang2016first,
    title     = {First-order methods for geodesically convex optimization},
    author    = {Hongyi Zhang and Suvrit Sra},
    booktitle = {Conference on Learning Theory, COLT},
    pages     = {1617--1638},
    year      = {2016},
}

@InProceedings{zhang2016riemannian,
    title     = {{R}iemannian {SVRG}: {F}ast stochastic optimization on {R}iemannian manifolds},
    author    = {Zhang, Hongyi and Reddi, Sashank J. and Sra, Suvrit},
    booktitle = {Advances in Neural Information Processing Systems, NeurIPS},
    pages     = {4592--4600},
    year      = {2016},
}

@InProceedings{zhang2018estimate,
    title     = {An estimate sequence for geodesically convex optimization},
    author    = {Hongyi Zhang and Suvrit Sra},
    booktitle = {Conference on Learning Theory, COLT},
    pages     = {1703--1723},
    year      = {2018},
}

@article{rumelhart1986learning,
    title={Learning representations by back-propagating errors},
    author={Rumelhart, David E and Hinton, Geoffrey E and Williams, Ronald J},
    journal={nature},
    volume={323},
    number={6088},
    pages={533--536},
    year={1986},
    publisher={Nature Publishing Group}
}

@inproceedings{kasai2019adaptive,
    author    = {Hiroyuki Kasai and
                 Pratik Jawanpuria and
                 Bamdev Mishra},
    title     = {Riemannian adaptive stochastic gradient algorithms on matrix manifolds},
    booktitle = {International Conference on Machine Learning, ICML},
    pages     = {3262--3271},
    year      = {2019},
}

@inproceedings{cho2017riemannian,
    author    = {Minhyung Cho and
                 Jaehyung Lee},
    title     = {Riemannian approach to batch normalization},
    booktitle = {Advances in Neural Information Processing Systems, NeurIPS},
    pages     = {5225--5235},
    year      = {2017},
}

@book {nesterov2004introductory,
    AUTHOR = {Nesterov, Yurii},
    TITLE = {Introductory lectures on convex optimization: A basic course},
    SERIES = {Applied Optimization},
    VOLUME = {87},
    PUBLISHER = {Kluwer Academic Publishers},
    YEAR = {2004},
    ISBN = {1-4020-7553-7},
    DOI = {10.1007/978-1-4419-8853-9},
}

@inproceedings{tripuraneni2018averaging,
    author    = {Nilesh Tripuraneni and
                 Nicolas Flammarion and
                 Francis Bach and
                 Michael I. Jordan},
    title     = {Averaging Stochastic Gradient Descent on Riemannian Manifolds},
    booktitle = {Conference on Learning Theory, COLT},
    pages     = {650--687},
    year      = {2018},
}

@article{ahn2020nesterov,
    title={From Nesterov's Estimate Sequence to Riemannian Acceleration},
    author={Ahn, Kwangjun and Sra, Suvrit},
    journal={arXiv preprint arXiv:2001.08876},
    year={2020}
}

@Article{lezcano2020curvature,
    title   = {Curvature-Dependant Global Convergence Rates for Optimization on Manifolds of Bounded Geometry},
    author  = {Lezcano-Casado, Mario},
    journal = {https://arxiv.org/abs/2008.02517},
    year    = {2020},
}

@misc{lecun1998mnist,
    title={The {MNIST} database of handwritten digits},
    author = {LeCun, Yann and Cortes, Corinna and Burges, Christopher C.J.},
    howpublished = {\url{http://yann.lecun.com/exdb/mnist/}},
    year={1998},
    note={Accessed: 2019-02-06}
}

@article {ambrose1953theorem,
    AUTHOR = {Ambrose, W. and Singer, I. M.},
     TITLE = {A theorem on holonomy},
   JOURNAL = {Transactions of the American Mathematical Society},
    VOLUME = {75},
      YEAR = {1953},
     PAGES = {428--443},
      ISSN = {0002-9947},
       DOI = {10.2307/1990721},
}
\clearpage
\onecolumn
\appendix
\section{Review of Vector Bundles and Connections}\label{sec:appendix_background}
In this section we summarize some standard results on theory of connections on vector bundles and get up to parallel transport systems and their relation with the usual definition of a covariant derivative. We choose to later use this framework rather than that of Riemannian metrics since all the problems that we encounter can already be seen in this more abstract framework, and one may see more geometrically the geometric obstructions to some constructions.

These results and definitions are all classic. A good reference for all of them is the book~\cite[][Chapters II, III]{kobayashi1963foundations}.

\paragraph{Vector bundles.}
A \emph{vector bundle $E$ of rank $k$ over $M$} consists of a pair of manifolds $E, M$ together with a submersion $\deffun{\pi : E -> M;}$ such that for every $x \in M$, the fibre $E_x \defi \pi^{-1}(\set{x})$ has a vector space structure, and for every $x \in M$ there exists a neighborhood $U \subset M$ and a linear isomorphism $t_x$ called a \emph{local trivialization} such that $\deffun{(\pi, t_x) :\pi^{-1}(U) -> U \times \RR^k;}$ is an isomorphism. This isomorphism indicates how to locally glue the fibers together. Intuitively, $E$ is a collection of vector spaces which represents the idea of having glued a vector space of dimension $k$ at every point of $M$. $E$ has locally the structure of product of manifolds. The tangent bundle $TM$ of a manifold $M$ of dimension $n$ is an example of a vector space of rank $n$ over $M$. The cylinder $S^1 \times \RR$ and the M\"obius bundle (a M\"obius band of infinite width) are two rank $1$ vector bundles over the circle $S^1$. The cylinder is a trivial bundle, as it is globally diffeomorphic to a product of spaces, while the M\"obius bundle is a non-trivial bundle, as it is just locally diffeomorphic to $U \times \RR$, $U \subset \RR$, but not globally. A \emph{section} of a vector bundle is a smooth map $\deffun{s : M -> E;}$ such that $\pi \circ s = \Id_M$. Intuitively, a section is a smooth choice of vectors on each vector space. When $E = TM$, a section is just a vector field. We denote by $\Gamma(E)$ the space of all sections on $E$.

\paragraph{Abstract and linear connections.}
The tangent bundle $TE$ of a vector bundle has a distinguished sub-bundle $VE \defi \ker \dif \pi$. In the same fashion that for vector spaces there is no natural inclusion of a quotient vector space onto the original vector space (it requires a choice of basis), here there is no natural complement to $VE$ in $TE$. We say that a sub-bundle $\HH\subset TE$ is an \emph{Ehresmann connection on $E$} if $VE \oplus \HH = TE$, in other words, for every $p \in E$, $V_pE$ and $\HH_p$ have zero intersection and generate the whole $T_pE$. As a consequence, we have that $\dif \pi(\HH_p) = T_{\pi(p)}M$ is a linear isomorphism. We say that an Ehresmann connection is a connection if it is linear on the fibers, that is, given the scalar multiplication on the fibers $\mu_a(p) = ap$, then $\pa{\dif\mu_a}_p\pa{\HH_p} = \HH_{ap}$.

\paragraph{Parallel transport systems.}
An equivalent way to talk about connections is through parallel transport systems. A \emph{parallel transport system} $\PP$ is an assignment for every curve $\deffun{\gamma : [a,b] -> M;}$ and point $p \in E_{\gamma(a)}$ of a curve $\PP_\gamma(p)$ on $E$ with initial point $p$ such that $\pi \circ \PP_\gamma(p) = \gamma$ such that certain linearity and regularity conditions hold. These conditions are: $1)$ For every curve $\gamma$, $\PPb_\gamma(p) \defi \PP_\gamma(p)(b)$ gives a linear isomorphism between $E_{\gamma(a)}$ and $E_{\gamma(b)}$ with inverse $\PPb^{-1}_\gamma = \PPb_{\gamma^{-}}$, where $\gamma^-$ is the curve walked in the reverse direction. $2)$ It is independent of the parametrization of $\gamma$. $3)$ Depends smoothly on $\gamma$ and $p$, and the initial tangent to the lifted curve only depends on the initial tangent vector to the curve on $M$.

When $\conn$ is the Levi-Civita connection on the tangent bundle $\deffun{\pi : TM -> M;}$ induced by a metric, the associated parallel transport system comes from parallel transporting vectors along curves.
The trivial parallel transport on a trivial vector bundle $E \iso M \times \RR^k$ for a point $(x, v) \in E$ and a curve such that $\gamma(0) = x$, is given by $\PP_\gamma(x,v)(t) = (\gamma(t), v) \in M \times \RR^k$. This gives the parallel transport system on $\RR^n$ with the canonical metric, since $T\RR^n \iso \RR^n \times \RR^n$.
One can prove that defining a parallel transport system is equivalent to defining a connection on a vector bundle. This makes explicit the idea that a connection is a way of identifying nearby tangent spaces on a manifold.

\paragraph{Covariant derivatives.}
A \emph{covariant derivative operator on $E$} is an application $\conn$ that lets us derive sections $s \in \Gamma(E)$, on a given direction $X \in TM$. We ask for this operator $\conn_X s$ to be linear in $X$ and Leibnitz in $s$, meaning that $\conn_X (fs) = X(f)s +f\conn_X s$. An example of a covariant derivative on $TM$ is the usual Levi-Civita connection given by a metric on $M$. Every a parallel transport system (or equivalently, every connection on $E$) gives raise to a covariant derivative operator on $E$ by the formula $\conn_{\dot{\gamma}(0)} s(\gamma(0)) \defi \frac{\dif}{\dif t}\vert_{t=0} \PPb^{-1}_{\gamma}\pa{s\pa{\gamma(t)}}$. Note that this notion coincides with that of the Levi-Civita connection, where $\conn_v X$ at a point $x \in M$ may be computed as the derivative of the parallel transport of $X(\gamma(t))$ back to $x$ where $\gamma(t)$ is the geodesic emanating from $x$ with $\gamma'(0) = v$. Covariant derivatives are a convenient way to perform computations, but they tend to hide the geometric intuition behind the ideas in differential geometry.

\clearpage
\section{Topological obstructions to admitting a connection with trivial holonomy}\label{sec:problems_further}
In this section we give a survey of some topological obstructions to admitting a flat connection and a flat metric. This comes to justify the assertion that the holonomy problems cannot be solved just by choosing a metric, as these problems are intrinsic to the topology of the manifold. We work in the context of vector bundles, as these are much more suited for this kind of questions than simply working on the tangent bundle. We give a short summary of these concepts in~\Cref{sec:appendix_background}.

Let $\deffun{\pi : E -> M;}$ be a vector bundle over $M$ together with a connection $\conn$. Associated to this connection, we have a parallel transport system $\PP$. Given that $\PPb^{-1}_\gamma = \PPb_{\gamma^-}$, we define the \emph{holonomy group} at a point $p$ as
\[
    \Hol(p) \defi \set{\PPb_{\gamma} | \gamma(a) = \gamma(b) = 0}.
\]

One may also define the \emph{restricted holonomy group} as
\[
    \Hol_0(p) \defi \set{\PPb_{\gamma} | \gamma(a) = \gamma(b) = 0, \gamma \text{ is null-homotopic}}.
\]
That is, $\Hol_0(p)$ is the group composed just by the isomorphisms given by curves that are contractible to a point.

One can prove that $\Hol(p)$ is a Lie group and $\Hol_0(p)$ is the connected component containing the identity. More importantly, we have that if the group is trivial, that is, $\Hol(p) = \set{\Id}$, then the parallel transport is also trivial. A trivial parallel transport system gives a global linear isomorphism between the fibers of the vector bundle, meaning that the vector bundle is globally trivial, that is, $E = M \times \RR^k$. This imposes a topological restriction on the type of connections that a bundle admits. What this means for us is that most manifolds do not admit a connection with trivial holonomy on their tangent bundle. For example, it is easy to see that the sphere $S^2$ does not have a trivial bundle, by the hairy ball theorem, as it does not admit a constant non-zero vector field. In general, the tangent bundle of $S^n$ is non-trivial for $n \neq 1, 3, 7$. This is a very strong obstruction given by the differentiable structure of the manifold to finding connections with trivial holonomy.

Even if the manifold has a trivial tangent bundle, this may still not be enough for our purposes. For example, it is not difficult to show using left-invariant vector fields that the tangent bundle of a Lie group $G$ is trivial, that is $TG = G \times \glie$. As such, one may put a trivial connection on this bundle following this trivialization. This connection will of course have trivial holonomy. On the other hand, this connection has torsion, unless $G$ is Abelian. In particular, a direct computation shows that this connection has torsion $T(X,Y) = -[X,Y]$ for left-invariant vector fields. In other words, whenever $G$ is not Abelian this connection does not come from any metric. Remember that a Lie group (in particular a matrix Lie group) is Abelian if and only if it is the product of a Euclidean space and a torus. As such, for most interesting Lie groups like $\GL{n}, \SO{n}, \U{n}, \SU{n}, \SL{n},\ldots$ this trivialization does not induce a parallel transport system without holonomy that comes from a metric when working on dimensions larger than $2$.

More generally, if the manifold admits a metric of non-positive curvature, and it is simply connected, then it is diffeomorphic to $\RR^n$ through the Riemannian exponential map. These are called Hadamard manifolds.
On these manifolds, one may clearly put a flat metric---and hence with trivial holonomy given that it is simply connected---simply by pushing forward the flat metric from $\RR^n$ along the diffeomorphism. This was explored for the case of symmetric positive definite matrices in the context of medical imaging and numerical analysis in~\cite{arsigny2006log,arsigny2006geometric}.

We use a similar idea to this one in our approach to define momentum and adaptive methods on manifolds. In particular, we use that for every manifold, the exponential map is a diffeomorphism from a radially convex subset of $\RR^n$ onto almost all the manifold, therefore being able to make almost all the tangent bundle trivial, and being able to put a flat metric on almost all the manifold. That is the reason why this approach is called a trivialization. Note that if our manifold were the symmetric positive definite matrices, the trivialization map would be a translation by left-multiplication by a matrix of that considered in~\cite{arsigny2006geometric}, so trivializations can be seen as a generalization of this paper to arbitrary manifolds.

\clearpage
\section{Experiment Set-Up and Hyperparameters}\label{sec:hyper}
The experiments tried to replicate faithfully the results of the previous papers. For that reason, we used the implementation of~\cite[][Section F]{lezcano2019trivializations} as a base.

The initialization, size, optimizers, and mixing constants of the adaptive terms in the algorithms, were the same as those described in previous papers across all the experiments. The only hyperparameters that were tuned were the learning rates of the orthogonal and the non-orthogonal parameters for the comparison to be fair.



The learning rates used were as follows:

\begin{table}[!ht]
    \centering
    \caption{Hyperparameters for \atriv{}$1$.}
    \label{tab:atriv1}
    \begin{tabular}{ll|cccc}
        \toprule
        Dataset & Size & Optimizer & Learning Rate & Orth.\ Optimizer & Orth.\ Learning Rate \\
        \midrule
        \midrule
        \multirow{3}{*}{\mnist} &
        $170$ &
        \multirow{6}{*}{\rmsprop} &
        $7\cdot 10^{-4}$ &
        \multirow{6}{*}{\rmsprop} &
        $10^{-4}$ \\
        & $360$ & & $5 \cdot 10^{-4}$ & & $10^{-4}$ \\
        & $512$ & & $3 \cdot 10^{-4}$ & & $5 \cdot 10^{-5}$ \\
        \multirow{3}{*}{\pmnist} &
        $170$ &
        &
        $7\cdot 10^{-4}$ &
        &
        $10^{-4}$ \\
        & $360$ & & $7 \cdot 10^{-4}$ & & $7 \cdot 10^{-5}$\\
        & $512$ & & $5 \cdot 10^{-4}$ & & $7 \cdot 10^{-5}$ \\
        \midrule
        \multirow{3}{*}{\timit} &
        $224$ &
        \multirow{3}{*}{\adam} &
        $10^{-3}$ &
        \multirow{3}{*}{\rmsprop} &
        $2 \cdot 10^{-4}$ \\
        & $322$ & & $10^{-3}$ & & $2 \cdot 10^{-4}$ \\
        & $425$ & & $10^{-3}$ & & $2 \cdot 10^{-4}$ \\
        \bottomrule
    \end{tabular}
\end{table}

\end{document}